\documentclass[11pt,a4paper]{article}
\usepackage{times,latexsym}
\usepackage[T1]{fontenc}

\PassOptionsToPackage{breaklinks}{hyperref}
\usepackage[acceptedWithA]{tacl2018v2}
\usepackage{tacl2018v2}
\usepackage{xurl} 

\usepackage{amsmath,amssymb}
\DeclareFontFamily{U}{rcjhbltx}{}
\DeclareFontShape{U}{rcjhbltx}{m}{n}{<->rcjhbltx}{}
\DeclareSymbolFont{hebrewletters}{U}{rcjhbltx}{m}{n}
\DeclareMathSymbol{\mem}{\mathord}{hebrewletters}{109}

\usepackage{amsmath}
\usepackage{amsthm}
\usepackage{amssymb}
\usepackage{bbm}
\usepackage{mathtools}
\usepackage{paralist}

\usepackage{mathtools}

\usepackage{csquotes}

\DeclarePairedDelimiter\abs{\lvert}{\rvert}%

\DeclareFixedFont{\ttb}{T1}{txtt}{bx}{n}{8} 
\DeclareFixedFont{\ttm}{T1}{txtt}{m}{n}{8}  

\usepackage{color}
\definecolor{deepblue}{rgb}{0,0,0.5}
\definecolor{deepred}{rgb}{0.6,0,0}
\definecolor{deepgreen}{rgb}{0,0.5,0}

\usepackage{listings}
\usepackage{tabu}

\newcommand\pythonstyle{\lstset{
language=Python,
basicstyle=\ttm,
otherkeywords={self},             
keywordstyle=\ttb\color{deepblue},
emph={MyClass,__init__},          
emphstyle=\ttb\color{deepred},    
stringstyle=\color{deepgreen},
commentstyle=\color{deepred}\sffamily,
frame=tb,                         
showstringspaces=false,            %
mathescape=true
}}

\lstnewenvironment{python}[1][]
{
\pythonstyle
\lstset{#1}
}
{}

\newcommand\pythoninline[1]{{\pythonstyle\lstinline!#1!}}

\DeclareFontFamily{U}{rcjhbltx}{}
\DeclareFontShape{U}{rcjhbltx}{m}{n}{<->rcjhbltx}{}
\DeclareSymbolFont{hebrewletters}{U}{rcjhbltx}{m}{n}
\DeclareMathSymbol{\mem}{\mathord}{hebrewletters}{109}
\DeclareMathSymbol{\mem}{\mathord}{hebrewletters}{109}

\newfam\hebfam
\font\tmp=rcjhbltx at10pt \textfont\hebfam=\tmp
\font\tmp=rcjhbltx at7pt  \scriptfont\hebfam=\tmp
\font\tmp=rcjhbltx at5pt  \scriptscriptfont\hebfam=\tmp

\edef\declfam{\ifcase\hebfam 
     0\or1\or2\or3\or4\or5\or6\or7\or8\or9\or A\or B\or C\or D\or E\or F\fi}

\mathchardef\shin   = "0\declfam 98

\newcommand{\assert}{\aleph}

\usepackage{stmaryrd}
\DeclarePairedDelimiter{\den}{\llbracket}{\rrbracket}

\theoremstyle{definition}
\newtheorem{definition}{Definition}
\theoremstyle{plain}
\newtheorem{theorem}{Theorem}

\newtheorem{question}{Question}

\usepackage{xpatch}
\makeatletter
\AtBeginDocument{\xpatchcmd{\@thm}{\thm@headpunct{.}}{\thm@headpunct{}}{}{}}
\makeatother

\usepackage{xspace}

\usepackage{centernot}

\newcommand{\rel}{\circ}

\DeclareMathOperator{\supp}{\mathrm{supp}}
\newcommand{\suppinv}{\supp^{{-}1}_L}

\newcommand{\upd}[1]{#1}
\newcommand{\uffda}[1]{#1}

\usepackage{xspace,mfirstuc,tabulary}

\newif\iftaclinstructions
\taclinstructionsfalse 
\iftaclinstructions
\renewcommand{\confidential}{}
\renewcommand{\anonsubtext}{(No author info supplied here, for consistency with
TACL-submission anonymization requirements)}
\newcommand{\instr}
\fi

\iftaclpubformat 

\else

\fi

\title{
Provable Limitations of Acquiring Meaning from Ungrounded Form:\\
What Will Future Language Models Understand?
}

\author{
    William Merrill\footnotemark[1] \quad Yoav Goldberg\footnotemark[1]\;\;\footnotemark[2] \quad Roy Schwartz\footnotemark[3] \quad Noah A. Smith\footnotemark[1]\;\;\footnotemark[4] \\
    \footnotemark[1]\;\;Allen Institute for AI \quad
    \footnotemark[2]\;\;Bar Ilan University \\
    \footnotemark[3]\;\;Hebrew University of Jerusalem \quad
    \footnotemark[4]\;\;University of Washington \\
  {\texttt{\{willm,yoavg,roys,noah\}@allenai.org}} \\
}

\date{\today}

\begin{document}
\maketitle
\begin{abstract}
Language models trained on billions of tokens have recently led to unprecedented results on many NLP tasks. This success raises the question of whether, in principle, a system can ever ``understand'' raw text without access to some form of grounding. We formally investigate the abilities of ungrounded systems to acquire meaning. Our analysis focuses on the role of ``assertions'': textual contexts that provide indirect clues about the underlying semantics. We study whether assertions enable a system to emulate representations preserving semantic relations like equivalence. We find that assertions enable semantic emulation of languages that satisfy a strong notion of semantic transparency. However, for classes of languages where the same expression can take different values in different contexts, we show that emulation can become uncomputable. Finally, we discuss differences between our formal model and natural language, exploring how our results generalize to a modal setting and other semantic relations. Together, our results suggest that assertions in code or language do not provide sufficient signal to fully emulate semantic representations. We formalize ways in which ungrounded language models appear to be fundamentally limited in their ability to ``understand''.
\end{abstract}

\section{Introduction} \label{sec:intro}
Recently, language models trained on huge datasets of raw text have pushed the limits of natural language processing \citep[among others]{devlin-etal-2019-bert, raffel2019exploring, brown2020language}. Such systems transcend the \emph{expert system} paradigm, where rules about language and meaning are hardcoded into a system, as well as the \emph{supervised learning} paradigm, where a notion of meaning is provided through ground-truth labels. Rather, analysis of massive language models has revealed that, to some degree, knowledge of syntactic and semantic dependencies can emerge \textit{without explicit supervision} \citep{rogers2020primer, tenney-etal-2019-bert}. This knowledge can then be transferred to a variety of downstream NLP tasks.

Yet, today's NLP systems built on large language models still fall short of human-level general understanding \citep{yogatama2019learning, zhang-etal-2020-winowhy}.
\citet{brown2020language} discuss the limitations of their GPT-3 language model compared to humans, suggesting that:

\begin{displayquote}
Scaling up any LM-like model \ldots\ may eventually run into (or could already be running into) the limits of the pretraining objective.
\end{displayquote}

\noindent This possibility raises an interesting theoretical question. What are the fundamental limits of learning meaning from language modeling, even assuming a perfect learner with access to unlimited data?
Recently, \citet{bk-2020} argued that achieving true natural language understanding from text alone is impossible, and that, to really get at meaning, some type of semantic grounding is necessary.\footnote{See \citet{blog2020} for a summary of the informal discussion around \citet{bk-2020}, much of which took place on social media.}
Their style of argumentation largely focused on developing thought experiments, rather than making formal arguments.



One thought experiment featuring prominently in \citet{bk-2020} was the task of learning to understand a programming language's semantics from raw code. Here, understanding was defined as fully emulating a compiler. This setup has clear parallels to learning to understand natural language, although the more well-defined nature of programming languages makes them easier to reason about. \citet{bk-2020} argue that emulation is difficult in this setting, and perhaps impossible, because the source code alone contains no information about how it should be interpreted to create outputs. One counterpoint raised by the paper, as well as others \citep{blog2020, potts2020}, is the existence of unit tests, with \textit{assertions} encoding examples of input/output pairs for blocks of code.\footnote{\upd{Unit tests are blocks of code in a software project that are designed to test whether the core code is behaving correctly.}} For example, systematically observing blocks like \pythoninline{x = 3; assert x == 3} could let a system bootstrap the semantics of variable assignment, because a programmer is likely to write assertions that will pass. \upd{These assertions constitute a form of implicit grounding embedded within language modeling by the pragmatic concerns of programmers, and they could potentially be leveraged to emulate a compiler.}\footnote{\upd{Contexts like assertions can be seen as an argument in favor of the distributional hypothesis \cite{harris1954distributional}.}}  However, it is not immediately clear if unit tests provide ``enough'' supervision to do this, even with unlimited data.


Viewing the debate about the power of assertions as central to the larger philosophical question, we aim to clarify it in more formal terms.
In this paper, we formally study whether observing a generalized notion of assertions can allow a system to ``understand'' strings. An assertion is a query about whether two strings evaluate to the same value within a fixed context.
This is motivated by the role of assertions in unit tests, where asserting two expressions are equal suggests that they have the same value within the test.

While assertions are directly motivated by the compiler thought experiment, they also have analogs in natural language, where sentences make assertions about the world, and it is reasonable to expect \upd{some form of} bias towards true statements \citep{potts2020}. Indeed, this is one of Grice's Maxims \citep{grice1975logic}: a set of basic principles proposed to govern the pragmatics of natural language.
For example, the truth conditions of \emph{This cat is the cat that Mary owns} verify that two cats in the world identified in distinct ways are the same entity.
In general, we might expect a sentence to appear with higher frequency if its truth conditions hold within its context, similar to an assertion in code, although of course there will also be other factors governing sentence frequency besides this.
In this sense, the example sentence resembles the Python statement \pythoninline{assert cat1 == cat2}, where \pythoninline{cat1} and \pythoninline{cat2} are two \pythoninline{Cat} objects.
See \autoref{sec:towards-nl} for more discussion of how assertions and other formal concepts translate to natural language.
We will generalize assertions to an abstract formal language context, allowing us to study how they can be used to emulate semantic relations.

Our findings are as follows. If every expression in a language has the same value in every valid context,
then the language can be emulated using a finite number of assertion queries (\autoref{sec:no-side-effects}). However, we construct a class of languages where expressions can take different values in different contexts, and where assertions do not enable emulation, i.e., infinite queries would be required (\autoref{sec:side-effects}). Intuitively, this means that assertions do not provide enough signal for a Turing-complete emulator to fully ``understand'' languages from this class. We go on to discuss differences between our formal model and the less well-defined context of natural language (\autoref{sec:towards-nl}).
These results provide a formal way to characterize upper bounds on whether it is possible to emulate the semantics of a language from distributional properties of strings.
\upd{Within our framework, in certain settings, we find that meaning cannot be learned from text alone.}
We strengthen claims made by \citet{bk-2020} that assertions in code do not necessarily provide sufficient signal for a language model to emulate understanding.
We do not make strong claims about how these results transfer to natural language, although we expect that the added complexity of natural language would make it, if anything, more difficult to ``understand'' than code.\footnote{\autoref{sec:old-emulation} documents and motivates conceptual changes since the original arXiv version of the paper.}
\section{Preliminaries}

Let $L \subseteq \Sigma^\star$ denote a formal language over alphabet $\Sigma$.
We will use $\lambda$ to denote the empty string.

Let $(\Sigma^\star)^2$ denote the Cartesian product of $\Sigma^\star$ with itself; i.e., the set of all pairs of strings.
Resembling \citet{clark2010three}, we refer to a tuple $\langle l, r \rangle \in (\Sigma^\star)^2$ as a syntactic \emph{context}. We also use other symbols to refer to a context, e.g., $\kappa = \langle l, r \rangle$. We denote by $\lambda^2$ the empty context $\langle \lambda, \lambda \rangle$.

\subsection{Meaning}
We will model formal languages not just as sets of strings, but as having an associated semantics.\footnote{\uffda{We slightly abuse notation by using $L$ to refer to both a set of strings, and a set of strings paired with a denotation function, which could be written more verbosely as $\langle L, \den{\cdot}_L \rangle$.}}
Specifically, we assume the existence of a \emph{denotational semantics} over every substring of $L$, which we now elaborate on.
Let $Y$ be a countable set of referents.
First, we will say that some $e \in \Sigma^\star$ is a valid \emph{expression} within the context $\kappa = \langle l, r \rangle$ if there exists some contextual denotation $\den{e \mid \kappa}_L \in Y$. Intuitively, this represents the value of $e$ when it occurs in the larger context $ler \in L$. We will also use the notation $\den{e \mid l, r}_L$ where convenient.
We will reserve $\emptyset \in Y$ as a special null symbol, defining $\den{e \mid \kappa}_L = \emptyset$ iff $e$ is not a valid expression in the context $\kappa$.\footnote{\upd{Our simple model of denotations does not reflect the full range of semantic theories that have been proposed for natural language. In particular, our denotations $\den{e \mid \kappa}_L$ depend only on the linguistic context $\kappa$ rather than any external world state. This differs substantially from how truth conditions are traditionally conceptualized in formal semantics \citep{Heim1998-HEISIG}. For example, in our framework, the referent of English $\den{\textit{the dog} \mid \kappa}_L$ must be fixed with no regard for the extralinguistic context. \autoref{sec:towards-nl} further contrasts our setup with the richer semantics of natural language.}}

Each context $\kappa \in (\Sigma^\star)^2$ also has a \emph{support}, or set of expressions that are valid within it:
\begin{equation*}
    \supp_L(\kappa) = \{ e \in \Sigma^\star \mid \den{e \mid \kappa}_L \neq \emptyset \} .
\end{equation*}

\paragraph{Example}
Let $L$ be a language of integers along with the \pythoninline{+} operator, e.g., \pythoninline{2 + 2}.
$Y$ is simply the integers.
We take $\den{e \mid \kappa}_L$ to map $e$ to its standard arithmetic interpretation, i.e., $\den{\pythoninline{2 + 6} \mid \lambda, \pythoninline{+ 4}}_L = 8$.
We take expressions that are not conventionally well-formed to be invalid: e.g., $\den{\pythoninline{+} \mid \lambda, \pythoninline{+}}_L = \emptyset$.
Finally, let $\kappa = \langle \lambda, \pythoninline{+ 4}\rangle$. Then $\supp_L(\kappa) = L$, since any valid expression can occur within $\kappa$.

\subsection{Strong Transparency}

As defined above, we make very few assumptions about denotations.
They are not necessarily compositional, and expressions may take different referents in different contexts.
However, we saw in the integer expression language that the meanings of an expression did not depend on its context.
We now define a property formalizing this idea.



\begin{definition}[Strong transparency] \label{def:ref-trans}
$L$ is \emph{strongly transparent} iff,
for all $e \in \Sigma^\star$,
$\kappa \in (\Sigma^\star)^2$, either $\den{e \mid \kappa}_L = \den{e \mid \lambda^2 }_L \neq \emptyset$, or $\den{e \mid \kappa}_L = \emptyset$.
\end{definition}

\noindent Informally, strong transparency says each $e$ has a well-defined denotation that exists independent of context, and that this simple denotation can be ``plugged into'' any context.
Our previous example expression $\pythoninline{2 + 6}$ is strongly transparent because it can be said to have a well-defined value $8$ independent of its context. We could break strong transparency by adding bound variables to the language, e.g. $\pythoninline{x = 2; x + 6}$ in Python. In this case, $\den{\pythoninline{x} \mid \kappa}_L$ \upd{non-vacuously} depends on $\kappa$.

\uffda{
Strong transparency resembles referential transparency \citep{russell25}, but is a stronger condition, in that it does not allow the same name to \emph{ever} refer to different values. For example, for a Python program, strong transparency does not allow assigning local variables within a function, even if the function output would remain completely specified by its inputs.
}

\subsection{Assertion Queries}

We now define an oracle function providing assertion information about expressions in $L$, resembling \pythoninline{assert e1 == e2} for two Python expressions \pythoninline{e1}, \pythoninline{e2}. A system is granted access to this function, and it can make \textit{assertion queries} to it in order to learn about the semantics of $L$.\footnote{This resembles the role of queries in classical grammar induction works \citep[e.g.,][]{angluin1987queries}.}
An assertion query tells us whether two expressions $e, e'$ are equivalent within the context $\kappa$.

\begin{definition}[Assertion oracle]
For $e, e', \in \Sigma^\star$ and $\kappa \in (\Sigma^\star)^2$, define the \emph{assertion oracle}
\begin{equation*}
    \assert_L(e, e' \mid \kappa) =
    \begin{cases}
        1 & \textrm{if} \; \den{e \mid \kappa}_L = \den{e' \mid \kappa}_L \\
        0 & \textrm{otherwise} .
    \end{cases}
\end{equation*}
\end{definition}

Recall that we defined $\den{e \mid \kappa}_L = \emptyset$ if $e$ is not valid in the context $\kappa$. 
In our example language of integer expressions, for all $\kappa$,  $\assert_L(\pythoninline{4}, \pythoninline{2 + 2} \mid \kappa) = 1$, since $4 = 2 + 2$. The computational power of this oracle depends on the complexity of the underlying semantics: for arbitrary semantics, it can become uncomputable.
In this paper, though, we focus on classes of languages for which the denotation function and assertion oracle are computable.

The $\assert_L$ oracle is motivated by assertion statements in programming languages, which occur naturally in environments like unit tests.
The distribution of strings in a corpus of code should capture some notion of this oracle,
since a programmer is more likely to assert two expressions are equal if they are expected to have the same value.
Our goal is to study the limits of understanding achievable from raw text, so we consider an ``upper bound'' setup by assuming a system has full access to $\assert_L$.
Can the system use this powerful oracle to emulate the underlying semantics?

\subsection{Turing Machines}
Our notion of language understanding will be based around the idea of emulation, which in turn requires a model of computational realizability.
We will use Turing machines \citep{turing1936computable} as a model of universal computation.
We write $\mu(e)$ for the output of Turing machine $\mu$ evaluated on input $e \in \Sigma^\star$.
We will also define an oracle Turing machine as a standard Turing machine that can compute a blackbox ``oracle'' function $f$ as a subroutine. We imagine the machine has a special \emph{query} instruction and tape. After writing $x$ to the query tape and executing the query instruction, the query tape will contain $f(x)$.
We will write $\mu_f(e)$ for the Turing machine $\mu$ evaluated on input $e$ with oracle access to $f$.
In the case where $f = \aleph_L$, we will simply write $\mu_L(e)$.
Whereas, in computability theory, oracle Turing machines are generally leveraged to make reductions from uncomputable problems, here we will use them to formalize the ability of an emulator to make assertion queries about $L$. This oracle provides additional power because these queries contain additional information beyond that encoded in the input expression.

\section{Research Question: Do Assertions Enable Emulation?} \label{sec:question}

There is a long history in AI of trying to define and measure understanding. \citet{turing1950computing} constitutes an early \upd{behaviorist} perspective; more recent approaches tend to emphasize not just an external view of a system's behavior, but also ``how it is achieved'' \citep{levesque2014}. 
Understanding can be behaviorally diagnosed in neural models by evaluating them on benchmarks \citep{wang-etal-2018-glue}.
An alternate approach is probing \citep{adi2017fine, conneau-etal-2018-cram, ijcai2018-796, hewitt-liang-2019-designing, belinkov2019analysis}, which investigates \textit{how directly} a model's representations encode semantic relations by measuring if they can be easily decoded from them.
Similarly, 
we take the position that systems are capable of understanding if they \emph{emulate} representations that are isomorphic to underlying meaning under important semantic relations like equivalence. We will formalize this in \autoref{qu:computing-meaning}, which asks whether such emulation is possible using assertions.

\begin{figure}
    \centering
    \includegraphics[width=\columnwidth]{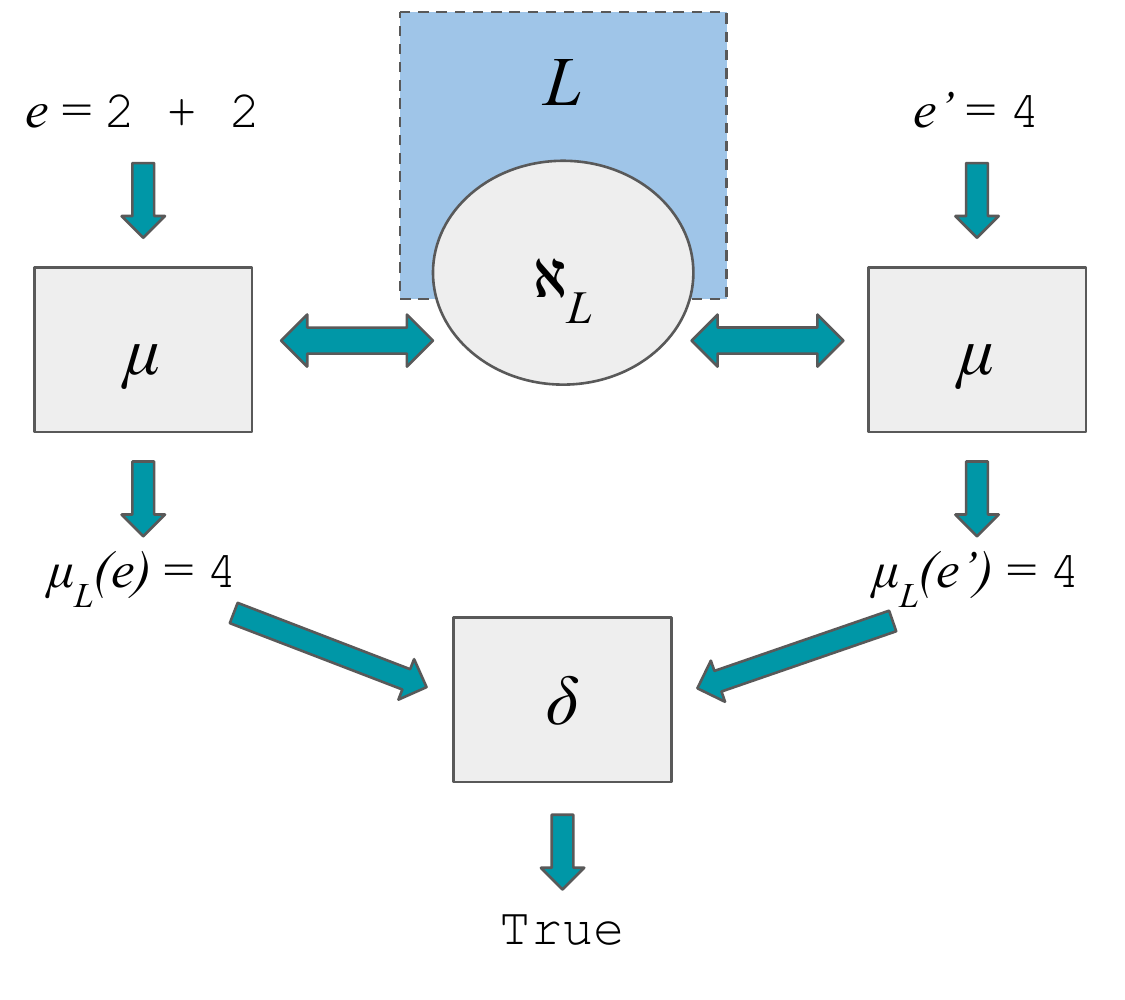}
    \caption{An illustration of \autoref{def:emulation}. $\mu$ emulates a representation of each expression using assertion queries. Then, $\delta$ compares the emulated representations to determine equivalence.}
    \label{fig:emulation}
\end{figure}

\begin{figure*}[ht!]
    \centering
    \begin{python}
from typing import Callable

AssertType = Callable[[str, str, str, str], bool]

def emulate(expr: str, asserteq: AssertType) -> int:
    for idx, cand in enumerate(all_strings()):
            if asserteq(expr, cand, "", ""):
                return idx
    \end{python}
    \caption{\pythoninline{emulate} implements an emulator $\mu$. Let \pythoninline{all\_strings} be an iterable enumerating all strings in $\Sigma^\star$. We provide a concrete implementation of \pythoninline{all\_strings} in \autoref{fig:subroutines}.}
    \label{fig:countable}
\end{figure*}

\uffda{
\begin{definition}[$\assert$-emulation] \label{def:emulation}
A class of languages $\mathcal L$ over $\Sigma$ is \emph{$\assert$-emulatable} if there exists an oracle Turing machine $\mu$ and standard Turing machine $\delta$ such that, for all $L \in \mathcal L$,
$\kappa \in (\Sigma^\star)^2$, and $e, e' \in \supp_L(\kappa)$,
\begin{align*}
    \den{e \mid \kappa}_L = \den{e' \mid \kappa}_L \iff  \delta \big( \mu_L(e), \mu_L(e') \mid \kappa \big) .
\end{align*}
\end{definition}
$\mu$ can be thought of as an emulator that evaluates expressions, whereas $\delta$ receives two values and decides whether they are equal. Crucially, only $\mu$ has direct access to $\assert_L$. $\delta$ can only use information from the oracle to the extent that it is encoded in the representations $\mu_L(e)$ and $\mu_L(e')$.}

\uffda{
\autoref{def:emulation} formulates emulation as a decision problem, as is typical in theoretical computer science. Equivalently, $\delta$ can be replaced by a computable function $\rho$ such that $\rho(\mu_L(e) \mid \kappa)$ \emph{evaluates} $\mu_L(e)$ in context $\kappa$, i.e., its output string is isomorphic to $\den{e \mid \kappa}_L$ under $=$. The functions $\delta$ and $\rho$ are Turing-reducible to each other, implying that if one definition is satisfied, so is the other.}

With our definition of emulation in place, we can formally state the research question:
\begin{question} \label{qu:computing-meaning}
For a class of languages $\mathcal L$, is $\mathcal L$ $\assert$-emulatable?
\end{question}
How does \autoref{qu:computing-meaning} relate to understanding in large language models? We imagine that, with sufficiently large amounts of data, the frequencies of strings in $L$ carry enough signal such that the language model objective ``supervises'' access to $\assert_L$.
Thus, $\mu_L(e)$ can be thought of as the language model representation of an expression $e$.
We then hope to recover underlying semantic relations from the representations produced by the language model via some function $\delta$.
The class $\mathcal L$ corresponds to a set of hypothesis languages over which the language model must search for the true $L$. We will see that whether emulation is possible will depend on the properties of $\mathcal L$.

Stepping back, \autoref{qu:computing-meaning} bears on the role of assertions raised by \citet{bk-2020}. Does observing assertions allow a Turing-complete system to emulate a compiler? \upd{In more general terms, are assertions powerful enough implicit grounding to achieve representations that encode the denotational semantics of a language?}
\section{Strong Transparency} \label{sec:no-side-effects}

\uffda{We first consider the case where the language being learned is known to be strongly transparent. Let \textsc{Transparent} denote the class of strongly transparent languages. We will show that \textsc{Transparent} is $\aleph$-emulatable.} The core idea of the proof is to construct a canonical form for each expression. The canonical form is the first expression in a lexicographic ordering that the assertion oracle deems equivalent to the target expression. For technical reasons, the emulator returns the index of this string under the lexicographic order.

\begin{theorem} \label{thm:transparency}
\textsc{transparent} is $\aleph$-emulatable.
\end{theorem}

\begin{proof}
As Python is Turing-complete, we write $\mu : \Sigma^\star \rightarrow \mathbb{N}$ as a Python function \pythoninline{emulate} in \autoref{fig:countable}.
The function receives as input an expression \pythoninline{expr} and a callback function \pythoninline{asserteq} to an oracle computing $\assert_L$. For each $e \in \Sigma^\star$, there exists $e^\star \in \Sigma^\star$ such that $\assert_L(e, e^\star \mid \lambda^2) = 1$. In the ``worst case'', this holds when $e^\star = e$ by symmetry. By construction, \pythoninline{all\_strings} reaches all strings in finite time. Therefore, the number of loop iterations before reaching $e^\star$ is finite. We can conclude that \pythoninline{emulate} halts on every $e \in \Sigma^\star$, establishing that it is computable.

Now, we move towards justifying that the emulation is correct for every $\kappa \in (\Sigma^\star)^2$.
We note that $\delta$ is simply the indicator function for equality over the natural numbers:
\begin{equation*}
    \delta(m, m' \mid \kappa) =
    \begin{cases}
        1 & \textrm{if} \; m = m' \\
        0 & \textrm{otherwise} .
    \end{cases}
\end{equation*}
The function \pythoninline{emulate} outputs $i \in \mathbb{N}$, the index of the first string $e^\star$ such that
$\den{e \mid \lambda^2 }_L = \den{e^\star \mid \lambda^2 }_L$.
Now, let $e, e' \in \supp_L(\kappa)$ be different inputs to $\mu$.
Because the enumeration order of the \pythoninline{for} loop is fixed across computation of $\mu_L(e)$ and $\mu_L(e')$:
\begin{align*}
    \mu_L(e) = \mu_L(e')
    &\iff \den{e \mid \lambda^2}_L = \den{e^\star \mid \lambda^2}_L \\
    &\quad \quad \wedge \den{e' \mid \lambda^2}_L = \den{e^\star \mid \lambda^2}_L \label{eq:pre-equiv-rel} \\
    &\iff \den{e \mid \lambda^2}_L = \den{e' \mid \lambda^2}_L \\
    &\iff \den{e \mid \kappa}_L = \den{e' \mid \kappa}_L ,
\end{align*}
\noindent where the last step follows by strong transparency.
We conclude that the conditions for emulation (\autoref{def:emulation}) are fully satisfied.
\end{proof}

Through a simple construction,
we have shown it is possible to emulate meaning from assertion queries for languages with strongly transparent semantics.
The number of bits in the emulated representation $\mu_L(e)$ is linear in the size of $e$.
In the next section, we consider what happens without strong transparency, where, among other complexities, values can be bound to variables, complicating the construction used in \autoref{thm:transparency}.
\section{General Case} \label{sec:side-effects}

\begin{figure*}[ht!]
\begin{minipage}{.5\textwidth}
\centering
\begin{python}
def leq() -> bool:
    return $\color{red}n$ < M
print(leq())
\end{python}
\end{minipage}%
\begin{minipage}{.5\textwidth}
\centering
\begin{python}
def leq() -> bool:
    return $\color{red}n$ < M
print(True)
\end{python}
\end{minipage}
\caption{Templates for strings in $L_m$, for $m \in \mathbb{N} \cup \{ \infty \}$. \pythoninline{M} evaluates to $m$ in all strings, while other expressions are evaluated according to Python 3.8 semantics. The metavariable $n$ ranges over $\mathbb{N}$ to form different strings in $L_m$, and is serialized as a decimal string.}
    \label{fig:template}
\end{figure*}

Requiring strong transparency precludes a broad class of linguistic patterns allowing an expression to refer to different values in different contexts.
For example, this includes assigning variable or function names in Python, or binding pronouns in natural language. These constructions can make emulation impossible to achieve from assertions. 
We will construct a class of languages based on Python where emulation is uncomputable.

\begin{definition} \label{def:template}
Let $\textsc{Leq} = \{ L_m \mid m \in \mathbb{N} \cup \{ \infty \} \}$, where strings in $L_m$ are defined according to \autoref{fig:template}.
For semantics, we first define $\den{\pythoninline{M} \mid \kappa}_{L_m} = m$.
For any other $ler \in L_m$ that is a well-formed Python 3.8 expression, we define $\den{e \mid l, r}_{L_m}$ as the value of $e$ assigned by the Python interpreter in the context $\langle l, r \rangle$.
For strings that are not valid Python expressions, define $\den{e \mid l, r}_{L_m} = \emptyset$.
\end{definition}

\uffda{
What does it take to emulate the expressions \pythoninline{leq()} and \pythoninline{True} in $L_m$?
If we knew $m$, then we could emulate them by simply comparing $n < m$.
However, it turns out that recovering $m$ for any $L_m \in \textsc{Leq}$ is not possible with a fixed number of assertion queries.
Formalizing this, we will show that \textsc{Leq} is not $\aleph$-emulatable.}\footnote{Another example of a non-$\aleph$-emulatable language takes \pythoninline{M} to be a finite list of integers and replaces \pythoninline{$\color{red} n\;$ < M} with \pythoninline{$\color{red} n \;$ in M}.}

\begin{theorem} \label{thm:side-effects}
\textsc{Leq} is not $\aleph$-emulatable.
\end{theorem}

\begin{proof}
\uffda{
Without loss of generality, we focus on the contexts for \pythoninline{leq()}\footnote{The only ``valid'' context for \pythoninline{leq()} is within \pythoninline{print($\cdot$)}. The denotation of \pythoninline{leq()} when it occurs next to \pythoninline{def} is $\emptyset$.} and \pythoninline{True} within \pythoninline{print($\cdot$)},
each of which is parameterized by some value of $n$.
Notationally, we identify each $L_m$ with $m$, and each context with its parameter $n$. This enables shorthand like
$\den{e \mid n}_m$ for the denotation of the expression $e$ in the context parameterized by $n$ in $L_m$.}

\uffda{
When $m = \infty$, it holds for all $n$ that $\assert_\infty(\pythoninline{leq()}, \pythoninline{True} \mid n) = 1$.
To satisfy emulation of $e \in \{\pythoninline{leq()}, \pythoninline{True}\}$,
$\mu_\infty$ makes a finite number of assertion queries
\begin{equation*}
    \assert_\infty(\pythoninline{leq()}, \pythoninline{True} \mid n_i) .
\end{equation*}
for some sequence of contexts $n_1, \cdots, n_q$, which we assume without loss of generality is sorted in increasing order.
We can adversarially construct $m' \neq \infty$ such that all these queries are the same, and thus $\mu_\infty(e) = \mu_{m'}(e)$ for both $e$. To implement this, we simply set $m' = n_q + 1$. Since $\mu_\infty(e) = \mu_{m'}(e)$, we conclude that, for all $n$,
\begin{align*}
    \delta( \mu_{m'}(\pythoninline{leq()}), \mu_{m'}(\pythoninline{True}) \mid n ) = \\ 
    \delta( \mu_\infty(\pythoninline{leq()}), \mu_\infty(\pythoninline{True}) \mid n ) .
\end{align*}
On the other hand, consider $n > n_q$. In this case,
\begin{align*}
    \den{\pythoninline{leq()} \mid n}_{m'} &= \pythoninline{False} \\
    \den{ \pythoninline{leq()} \mid n}_\infty &=  \pythoninline{True} ,
\end{align*}
which can be rewritten as
\begin{align*}
    \den{\pythoninline{leq()} \mid n}_{m'} &\neq \den{\pythoninline{True} \mid n}_{m'} \\
    \den{ \pythoninline{leq()} \mid n}_\infty &= \den{\pythoninline{True} \mid n}_\infty .
\end{align*}
Therefore, the conditions of $\aleph$-emulation (\autoref{def:emulation}) cannot be satisfied for both $L_{m'}$ and $L_\infty$. This implies that \textsc{Leq} is not $\aleph$-emulatable.
}
\end{proof}

\begin{figure*}
\begin{minipage}{.5\textwidth}
\centering
\begin{lstlisting}[mathescape=true,frame=tb,basicstyle=\small]
There is a number.
$\color{red} n$ is less than it.
\end{lstlisting}
\end{minipage}%
\begin{minipage}{.5\textwidth}
\centering
\begin{lstlisting}[mathescape=true,frame=tb,basicstyle=\small]
There is a number.
Zero equals one.
\end{lstlisting}
\end{minipage}
\caption{An informal construction adapting the program templates in \autoref{fig:template} to English. \upd{Under our framework, two sentences are considered equivalent if they are true in exactly the same set of contexts. If the number is allowed to be $\infty$, this cannot be done in general for the final lines of each template.}}
\label{fig:english-template}
\end{figure*}

\subsection{Discussion}

We briefly summarize this result in less formal terms.
\textsc{Leq} contains languages $L_m$ defined by \autoref{fig:template}.
Every program in each $L_m$ is easily computable. With knowledge of the Python interpreter and $m$, any agent could execute all of these programs.
\uffda{This can be formalized by observing that, for a fixed $m$, the class $\{ L_m \}$ is $\aleph$-emulatable.
Rather, what we have shown is that, with finite time, it is impossible for an ungrounded agent to emulate $L_m$ \textit{using assertion queries} when $m$ is unknown in advance. In other words, without prior knowledge of $m$, no algorithm can use assertions to disambiguate which notion of $=$ is used by $L_m$ from the infinite other possibilities.
In a rough sense, $m$ can be thought of as a cryptographic key enabling linguistic understanding: agents that know $m$ can directly emulate $L_m$, but agents without it cannot, at least using assertions.\footnote{
Alternatively, we can take a more complexity-theoretic perspective by measuring the number of queries needed to emulate up to a bounded context size. Fix a maximum $n$.
Then we can use binary search with $\mathcal O(\log n)$ queries to find the value of $m$.
Since the number of context bits is $\mathcal O(\log n)$, the numbers of queries is $\mathcal O(\abs{\kappa})$, beating the $\mathcal O(\abs{\Sigma}^{\abs{\kappa}})$ query complexity achievable by brute force.
This perspective somewhat resembles \citet{pratt-hartmann2006} and other work in semantic complexity theory on the computational complexity of evaluating fragments of natural language.
}

\uffda{\autoref{thm:side-effects} does not use the fact that $\delta$ must be computable, as opposed to an arbitrary function. Even if $\delta$ is an arbitrary function, it could not disambiguate whether $m$ halts based on queries.}}

It is more precise to state \autoref{thm:side-effects} in a formal language, but an argument similar to \autoref{thm:side-effects} can be adapted to a natural language like English. An example is shown in \autoref{fig:english-template}, where we define the meaning of a sentence as its truth conditions, and we imagine the class of candidate languages is formed by varying the unspecified \emph{number}, which can potentially be $\infty$.
\uffda{Deciding if \emph{$\color{red} n$ is less than it} has the same truth conditions as \emph{Zero equals one} is equivalent to comparing \pythoninline{leq()} and \pythoninline{True}. A system must necessarily fail to emulate the semantics of these expressions in some context, for some secret number.
The rest of the paper \upd{further explores} the implications and limitations of applying our formal model to natural language.}

\section{Towards Natural Language} \label{sec:towards-nl}

As discussed in \autoref{sec:intro}, our results are inspired by the thought experiment of whether a language model can use raw code to learn a compiler. A goal of this, of course, is to examine whether understanding can be acquired from natural language text in a simplified setting. In principle, our formal results can bear on this broader question about natural language, although some differences emerge when extending the results to a less well-defined setting. \upd{In many cases, these differences appear to make the task of learning meaning harder, suggesting that our negative claim in a simpler setting (\autoref{thm:side-effects}) may still hold as an impossibility result.} We now discuss some points of difference between our formal model and natural language.


\paragraph{Truth Conditions} \uffda{There are connections between our framework and the concepts of truth values and truth conditions in linguistic semantics. For a boolean-valued expression $e$, a truth value corresponds to computing $\den{e \mid \kappa}_L$ in a fixed context. On the other hand, truth conditions correspond roughly to a function computing $\den{e \mid \kappa}_L$ for any $\kappa$.
A crucial difference, though, is that these conditions cannot be \emph{intensional} \citep{von2011intensional}, i.e., they are not functions of the world state, but rather of the linguistic context only.
In this sense, emulation corresponds to recovering the ability to resolve non-intensional truth conditions of sentences.}
This model is natural for formalizing a closed programming language environment, e.g., with no environment variables or user input, \upd{since in this case the program state is specified completely by the linguistic context.} On the other hand, English has common elements like \textit{that} whose meaning can change depending on world state external to language. \upd{Perhaps allowing such elements would only make understanding more difficult; or, arguably, generally impossible, since there is no way for the model to observe the grounding world state using only an assertion oracle. We are inclined to believe that, since such changes would make understanding more difficult, \autoref{thm:side-effects} would still hold as an impossibility result. However, future work would be needed to make this idea precise.
}

\paragraph{Possible Worlds} \uffda{In the last paragraph, we discussed how mutable world state is an additional complexity of natural language compared to our setup. Similarly, speakers of natural languages have imperfect information about the world around them, which can be captured by modeling the referent of an expression over a set of \emph{possible} worlds, rather than within a specific evaluation context.}
In \autoref{sec:modal}, we explore to what degree this setting makes the task of learning to understand more difficult. In adapting our model to this context, the assertion oracle must become ``modal'' in the sense that it quantifies over sets of worlds. We explore two different models of modality for the oracle, corresponding to different physical interpretations. In one case, \autoref{thm:transparency} and \autoref{thm:side-effects} apply analogously, while, in the other, emulation becomes an ill-defined problem.

\paragraph{Denotation vs.~Intent} \citet{bk-2020} distinguish between \emph{standing meaning} and \emph{communicative intent}, reflecting a distinction between denotational semantics and other pragmatic intentions that a speaker has in producing an utterance. In this paper, it is most straightforward to take $\den{e \mid \kappa}_L$ to reflect standing meaning. In principle, we could imagine that it represents the speaker's communicative intent, and that an omniscient oracle $\assert_L$ can reveal information about the speaker's intents to the system. Even with this unrealistically powerful oracle, \autoref{thm:side-effects} says that the system cannot emulate the speaker's intents.

\paragraph{Competence vs.~Performance} \citet{chomsky2014aspects} differentiates competence and performance in linguistic theory, where competence corresponds roughly to the correct algorithmic modeling of a linguistic process, and performance describes its implementation subject to resource constraints like memory. Arguably, agents might be said to understand language if they are competent in this sense, even if they sometimes make performance errors. In contrast, our definition of emulation (\autoref{def:emulation}) permits no performance errors. In future work, it would be interesting to adapt an approximate notion of emulation that tolerates performance errors in order to more closely target understanding in a sense reflecting competence.

\paragraph{Other Relations} \autoref{thm:transparency} and \autoref{thm:side-effects} investigate whether $\assert_L$ can be used to emulate meaning representations that preserve an equivalence relation. While equivalence is an important part of semantics, other semantic relations like entailment are also necessary for language understanding. In \autoref{sec:other-relations}, we show a generalization of \autoref{thm:all-rel} extends to \emph{any} semantic relation. In other words, referential transparency also enables emulation of relations besides $=$.

\paragraph{Other Oracles} \uffda{
We believe assertions are a fairly general model of the types of semantics encoded in unsupervised learning resulting from a pragmatic bias for truth; however, it is possible other information is also represented, resulting from other pragmatic biases governing language usage and dataset creation. This additional information could be formalized as access to additional oracles. It would be exciting to formalize the power of multimodal setups by analyzing the interactions of oracles enabled by different input modalities.}
\section{Stepping Back}
In this work, we formalized an argument that was raised by \citet{bk-2020} as a thought experiment.
\citet{bk-2020} question whether unsupervised training objectives are the right goal to target for achieving natural language understanding. If meaning is defined as identifying which object in the real world, or which set of situations, a linguistic element refers to, then, in a direct sense, an ungrounded system cannot understand meaning. But \citet{bk-2020} go farther than this, claiming that an ungrounded system cannot even \emph{emulate} understanding because it is not clear how a system should learn to interpret strings, even if it can model their distribution. We formalize this idea of emulation as $\aleph$-emulation.


One counterargument mentioned by \citet{bk-2020} is that indirect forms of grounding do exist in programming and natural language, which we formalize as assertions. The syntactic distributions of statements like \pythoninline{assert} allow us to indirectly observe semantic relations over the denotations. Assertions are one way that the distribution of strings in a corpus is not blind to their semantics.
By studying them, we study whether this indirect grounding enables a computational system to emulate the underlying semantic relations.

\paragraph{Key Takeaways} While assertions allow a system to emulate semantic relations in simple cases where the semantics are referentially transparent, we find that linguistic constructs like variable binding bring this task in conflict with the fundamental laws of computability. \upd{In other words, under our formal model of meaning and emulation, it is not just intractable for an ungrounded system to emulate understanding of a formal language, but, in some cases, \emph{impossible}.} We provide constructive examples where understanding must necessarily break down. We present these results in a well-defined framework building off formal approaches in logic, linguistics, and computer science. While we do not prove anything about natural languages, we do show that ungrounded models must fail to emulate equivalence in a very simple setting. \upd{A similar result likely extends to natural language understanding as well, which among other things, requires modeling referential identity (e.g., for sentences like \emph{Manny is the cat}).}
Further, we believe much of our framework can be readily adopted in other works formalizing understanding in Turing-complete systems.

\paragraph{Open Questions}
In this work, we have focused on utterances, by default, as opposed to \emph{dialogues}. An exciting extension would be to formalize a dialogue between two speakers, interrupted by the ``octopus'' of \citet{bk-2020}.\footnote{\upd{The octopus thought experiment imagines a deep-sea octopus $O$ observes a dialogue between two humans by intercepting an underwater cable. Could $O$ learn to emulate the role of one of the speakers without exposure to life on land?}}
Existing theories of discourse could potentially be synthesized with this framework.
What linguistic properties besides referential transparency relate to emulatability?
Can this framework be extended to formalize multimodal setups, where multiple oracles from different domains can potentially be combined to gain additional power?
Finally, is there a natural way to relax our standard of emulation towards a probabilistic definition, and how would this change the results?



\section*{Acknowledgments}
We thank Mark-Jan Nederhof for his excellent suggestions. We also thank Dana Angluin, Matt Gardner, Eran Yahav, Zachary Tatlock, Kyle Richardson, Ruiqi Zhong, Samuel Bowman, Christopher Potts, Thomas Icard, and Zhaofeng Wu for their feedback on various versions of this work. Further thanks to our anonymous reviewers and researchers at the Allen Institute for AI and UW NLP. Finally, we appreciate the lively online discussion of the paper, which informed updates to the camera-ready version.

\bibliography{main}
\bibliographystyle{acl_natbib}

\appendix
\begin{figure*}
    \centering
    \begin{python}
from itertools import count, product
from typing import Iterable

def all_strings() -> Iterable[str]:
    for length in count():
        iterable = product(*[SIGMA for _ in range(length)])
        yield from ("".join(x) for x in iterable)
    \end{python}
    \caption{An concrete implementation of \pythoninline{all\_strings}, which is referenced in \autoref{fig:countable} and \autoref{fig:extended-countable}.}
    \label{fig:subroutines}
\end{figure*}



\section{Multiple Worlds} \label{sec:modal}

Programs execute in well-defined environments with a clear state. Speakers of natural language, on the other hand, have imperfect information and beliefs about the world around them. Thus, it can be more natural to model grounding context for language as a set of \emph{possible worlds}, rather than a single world state. We formalize this in two different ways (with two different physical interpretations) and explore how it affects our results.

Let $W$ be a set of all possible worlds.
We redefine denotations to be \textit{intensionalized} \citep{von2011intensional}, i.e., we write $\den{e \mid \kappa}^w$ as the denotation of $e$ in the context $\kappa$, evaluated in world $w \in W$.
Assume for simplicity that $Y = \{0, 1, \emptyset\}$.
We will now introduce modal denotations and assertions using a generic \emph{modal quantifier} $\odot$, which reduces a sequence of worlds to a boolean value according to some intensional predicate. This quantifier controls how multiple possible worlds are collapsed to form denotations and query outputs.

\begin{definition}[Modal denotation] \label{eq:modal-contextualized}
Let $\odot$ be a modal quantifier.
For all $e \in \Sigma^\star$, $\kappa \in (\Sigma^\star)^2$, define
\begin{equation*}
    \odot \den{e \mid \kappa}_L = \bigodot_{w \in W} \den{e \mid \kappa}_L^w .
\end{equation*}
\end{definition}
We will write the previously defined assertion oracle to apply in a specific world $w$, i.e. $\assert^w_L$. We also extend it to quantify over multiple worlds:

\begin{definition}[Modal assertion]
Let $\odot$ be a modal quantifier.
For all $e \in \Sigma^\star$, $\kappa \in (\Sigma^\star)^2$, define
\begin{equation*}
    \odot \assert_L(e, e' \mid \kappa) = \bigodot_{w \in W} \assert^w_L(e, e' \mid \kappa) .
\end{equation*}
\end{definition}

Specifically, we consider $\odot = \{\Box, \Diamond\}$, corresponding to universal and existential quantifiers over worlds. Thus, $\Box$ can be thought of as as $\forall$ over worlds, and $\Diamond$ can be thought of as $\exists$. For either quantifier, if any $\den{e \mid \kappa}_L^w = \emptyset$, we define $\odot \den{e \mid \kappa}_L= \emptyset$ as well.
Each quantifier will have a different physical interpretation. With universal quantification, we will find that results analogous to \autoref{thm:transparency} and \autoref{thm:side-effects} hold. With existential quantification, it turns out that the equivalence class of $\mu$ is underspecified. In other words, not only is it impossible to compute an emulator with a finite number of assertion queries, but, even with infinite assertions, there is no consistent way to emulate the underlying modal semantics.

\subsection{Universal Quantification}

In the first case we let $\odot = \Box$. Two expressions are viewed as having the same meaning if they are equivalent in every possible belief world.
This is interpretable as observing text $L_\Box$ written by a single author whose belief state is represented by multiple possible worlds. The author only asserts a statement is true if it is consistent across all worlds that they believe are possible.

In this setting, we will show that the modal assertion oracle uniquely specifies a modal denotation for each expression, up to isomorphism. In other words, as with the non-modal assertion oracle, each assertion query would let us decide some relation between two expressions. Thus, the same results for the non-modal setting discussed in the main body of the paper will also hold here.
\begin{theorem} \label{thm:well-formed}
Consider $e, e' \in \Sigma^\star$ and any context $\kappa \in (\Sigma^\star)^2$ such that $\Box \den{e \mid \kappa}_L \neq \emptyset$ and $\Box \den{e' \mid \kappa}_L \neq \emptyset$.
Then,
\begin{align*}
    \Box \den{e \mid \kappa}_L = \Box \den{e' \mid \kappa}_L \iff
    \Box \assert_L(e, e' \mid \kappa) .
\end{align*}
\end{theorem}
\begin{proof}
\begin{align*}
\Box \den{e &\mid \kappa}_L = \Box \den{e' \mid \kappa}_L \\
    &\iff \bigwedge_{w \in W} \den{e \mid \kappa}^w_L = \bigwedge_{w \in W} \den{e' \mid \kappa}^w_L \\
    &\iff \bigwedge_{w \in W} \big( \den{e \mid \kappa}^w_L = \den{e' \mid \kappa}^w_L \big) \\
    &\iff \bigwedge_{w \in W} \assert^w_L(e, e' \mid \kappa) \\
    &\iff \Box \assert_L(e, e' \mid \kappa) .
\end{align*}
\end{proof}
Crucial to this simple proof is the fact that $\wedge$ is distributive over $=$. This is specific to the quantifier being $\Box$.
\autoref{thm:well-formed} implies that $\Box \den{e \mid \kappa}_L$ can be recovered from modal assertion queries analogously to the non-modal case.
Thus, results analogous to \autoref{thm:transparency} and \autoref{thm:side-effects} apply for emulating $\Box \den{e \mid \kappa}_L$ using queries to $\Box \assert_L$.

\subsection{Existential Quantification}

In the second case we let $\odot = \Diamond$. Two expressions are viewed as having the same meaning if they are equivalent in \emph{some} world. This is interpretable as observing a large dataset of text $L_\Diamond$ generated by many authors, each with a different single belief world $w$. In the corpus, we imagine two expressions can be asserted to be equivalent in some context if \emph{any} of the authors would consider them to be equal in that context.

In this case, assertions do not even fully specify equivalence between the modal denotations. This is a stronger sense in which meaning cannot be emulated from assertion queries. Emulation is not just impossible with finite assertions, but mathematically underspecified.

\begin{theorem}
There exist $e, e' \in E(L)$ and $\kappa \in (\Sigma^\star)^2$ such that $\Diamond \den{e \mid \kappa}_L \neq \emptyset$ and $\Diamond \den{e' \mid \kappa}_L \neq \emptyset$,
and also $\Diamond \assert_L(e, e' \mid \kappa) = 1$ is consistent with either $\Diamond \den{e \mid \kappa}_L = \Diamond \den{e' \mid \kappa}_L$ or $\Diamond \den{e \mid \kappa}_L \neq \Diamond \den{e' \mid \kappa}_L$.
\end{theorem}

\begin{table}[t]
    \centering
    \begin{tabu}{|c|cc|c|[2pt]cc|c|}
        \hline
        & $e_1$ & $e_2$ & $\assert$ & $e_1$ & $e_2$ & $\assert$ \\
        \hline
        $w_1$ & $0$ & $0$ & $1$ & $0$ & $0$ & $1$\\
        $w_2$ & $0$ & $0$ & $1$ & $0$ & $1$ & $0$ \\
        \hline
        $\Diamond$ & $0$ & $0$ & $1$ & $0$ & $1$ & $1$ \\
        \hline
    \end{tabu}
    \caption{Two tables (separated by a thick line) representing two different versions of $W$. Within each table, each cell $i, j$ in the main 2-by-2 grid contains the boolean value $\den{e_j \mid \kappa}^{w_i}_L$. The column to the right contains $\assert^{w_i}_L(e_1, e_2 \mid \kappa)$. The bottom row aggregates each column by quantifying $\Diamond$.}
    \label{tab:ambiguity}
\end{table}

\begin{proof}
We construct an example with expressions $e_1, e_2$ in a single context $\kappa$. Fix $W = \{w_1, w_2\}$. \autoref{tab:ambiguity} shows two versions of this modal setup.
In both versions of the universe,  $\Diamond \assert_L(e, e' \mid \kappa) = 1$.
However, on the left, $\Diamond \den{e \mid \kappa}_L = \Diamond \den{e' \mid \kappa}_L$, while, on the right, the opposite holds. So, with $\Diamond$, modal assertions do not uniquely determine equivalence of modal denotations.
\end{proof}

As an equivalence class for $\mu$ is not even well-defined by $\Diamond \assert_L$, we cannot hope to compute it from queries. This is an even stronger sense in which emulation is impossible using assertions.
On some level, this may be a natural model for language modeling corpora, which aggregate text from potentially inconsistent sources.

In summary, if assertions uniquely determine equivalence between denotations in a strongly transparent language, then we can expect to emulate representations preserving equivalence using assertions. Otherwise, there are various levels of formal challenges to emulating equivalence.
\begin{figure*}[ht!]
    \centering
    \begin{python}
from typing import Callable, Dict, Tuple
    
AssertType = Callable[[str, str, str, str], bool]

def emulate(expr: str, assertrel: AssertType) -> Dict[Tuple[str, str], bool]:
    repres = {}
    for cand in all_strings():
        repres[expr, cand] = assertrel(expr, cand, "", "")
        repres[cand, expr] = assertrel(cand, expr, "", "")
        if expr == cand:
            return repres
    \end{python}
    \caption{\pythoninline{emulate} computes a structured representation of the input string \pythoninline{expr} that preserves any semantic relation $\rel$ in terms of assertion queries. The iterable \pythoninline{all\_strings} is defined in \autoref{fig:subroutines}.}
    \label{fig:extended-countable}
\end{figure*}

\section{Other Semantic Relations} \label{sec:other-relations}


Sections \ref{sec:no-side-effects}, \ref{sec:side-effects}, and \ref{sec:modal} investigate whether $\assert_L$ can be used to emulate meaning representations that preserve semantic equivalence. While equivalence is an important part of semantics, other semantic relations are also necessary for language understanding. For example, the following feature prominently in theories of linguistic semantics:
\begin{compactitem}
    \item \textbf{Entailment} In general terms, an entailment \citep{winter_2016} relation $\rightarrow$ is a partial order over $Y$. Intuitively, if $y \rightarrow y'$, then $y$ is a ``special case'' of $y'$. For example, one could construct $E$, a semantic analysis of English, where
    $\den{\textit{fat cat} \mid \textit{a}, \textit{sits}}_E \rightarrow \den{\textit{cat} \mid \textit{a}, \textit{sits}}_E$.
    \item \textbf{Contrary negation} Negation is a complex topic in semantics. One sense of negation is if two meaning representations are ``contrary'' \citep{sep-negation}, meaning both cannot be true at the same time.
\end{compactitem}
Does \autoref{thm:side-effects} generalize to other relations besides $=$? To answer this, we first extend assertions and emulation to apply to a generic relation $\circ : M^2$. The proof for \autoref{thm:transparency} does not fully translate to this new setting, but we will show via a new argument that emulation is still possible.
\begin{definition}
For $e, e', \in \Sigma^\star$ and $\kappa \in (\Sigma^\star)^2$, define the \emph{assertion oracle}
\begin{equation*}
    \assert_{L,\circ}(e, e' \mid \kappa) =
    \begin{cases}
        1 & \textrm{if} \; \den{e \mid \kappa}_L \circ \den{e' \mid \kappa}_L \\
        0 & \textrm{otherwise} .
    \end{cases}
\end{equation*}
\end{definition}
\begin{definition} \label{def:emulation++}
A class of languages $\mathcal L$ over $\Sigma$ is $\aleph$-emulatable w.r.t.~$\circ$ if there exists an oracle Turing machine $\mu$ and standard Turing machine $\delta$ such that, for all $L \in \mathcal L$,
$\kappa \in (\Sigma^\star)^2$, and $e, e' \in \supp_L(\kappa)$,
\begin{equation*}
    \den{e \mid \kappa}_L \circ \den{e' \mid \kappa}_L \; \iff \; \delta \big( \mu_L(e), \mu_L(e') \mid \kappa \big) .
\end{equation*}
\end{definition}
We now are ready to prove the extended form of \autoref{thm:transparency}. The main idea of the proof will be to memoize the value of the relation $\circ$ between $\den{e \mid \kappa}_L$ and the values of all expressions smaller than $e$. This guarantees that $\delta$ will be able to ``look up'' the correct output.
\begin{theorem} \label{thm:all-rel}
\textsc{transparent} is $\aleph$-emulatable w.r.t.~$\circ$.
\end{theorem}
\begin{proof}
Similarly to \autoref{thm:transparency}, we present the proof constructively as a Python program to compute $\mu$. We then show how to define $\delta$ appropriately, completing the proof.

\autoref{fig:extended-countable} shows the algorithm to compute $\mu_L(e) \in M$. In Python, $\mu_L(e)$ is a dictionary; we interpret it as a function $\mu_L(e) : \Sigma^\star \times \Sigma^\star \rightarrow \{0, 1, \emptyset\}$, where $\emptyset$ represents values that are not set. We define $\delta$ as follows:
\begin{align*}
    \delta(m, m' \mid \kappa) \iff
    &\begin{cases}
        m(e, e') & \textrm{if} \; m(e, e') \neq \emptyset \\
        m'(e, e') &
        \textrm{otherwise.}
    \end{cases}
\end{align*}

Crucially, it must be that $\mu_L(e)(e, e') \neq \emptyset$ or $\mu_L(e')(e, e') \neq \emptyset$.
In \autoref{fig:extended-countable}, \pythoninline{cand}
either reaches $e$ before $e'$, or $e'$ before $e$. By symmetry, assume it reaches $e$ before $e'$. Then $\mu_L(e')(e, e') \neq \emptyset$, so 
\begin{align*}
    \delta(\mu_L(e), \mu_L(e') \mid \kappa) &\iff \mu_L(e')(e, e') = 1 \\
    &\iff \assert_{L,\rel}(e, e' \mid \lambda^2) = 1 \\
    &\iff \den{e \mid \lambda^2} \rel \den{e' \mid \lambda^2} \\
    &\iff \den{e \mid \kappa} \rel \den{e' \mid \kappa} .
\end{align*}
\noindent Therefore \pythoninline{emulate} satisfies \autoref{def:emulation}.
\end{proof}

We needed to change the proof of \autoref{thm:all-rel} compared to \autoref{thm:transparency} because $\rel$ is not an equivalence relation. In \autoref{thm:transparency}, the final steps relied on reflexivity, transitivity, and symmetry: the three properties that constitute equivalence. The new proof enlarges the size of the emulated representations. Rather than representing each $e$ with a number, $\mu_L(e)$ becomes a large dictionary of strings. This represents an increase in space complexity from linear to exponential in the size of $e$.
\section{Old Emulation Definition} \label{sec:old-emulation}

A \href{https://arxiv.org/abs/2104.10809v1}{previous version} of this paper defined emulation slightly differently. We discuss the differences and explain the advantages of the new definition. First, we defined a \emph{general denotation} as
\begin{equation*} \label{eq:gen-meaning}
    \den{e}_L = \{ \langle \kappa, \den{e \mid \kappa}_L \rangle \mid \kappa \in (\Sigma^\star)^2 \} .
\end{equation*}
The general meaning represents the meaning of a word across all contexts. Now, say that two functions $f,g$ are isomorphic (with respect to $=$) over a set $X$ iff, for all $x, x' \in X$,
\begin{equation*}
    f(x) = f(x') \iff g(x) = g(x') .
\end{equation*}
We will write $f \cong_= g$ in this case.
We will refer to a set of contexts $S \subseteq (\Sigma^\star)^2$ as a \emph{syntactic role}. Each syntactic role has a set of expressions $\suppinv(S)$ whose \emph{support} is that role:
\begin{align*}
    \supp_L(e) &= \{ \kappa \in (\Sigma^\star)^2 \mid \den{e \mid \kappa}_L \neq \emptyset \} \\
    \suppinv(S) &= \{ e \in \Sigma^\star \mid \supp_L(e) = S \} .
\end{align*}
We can now give the old definition of emulation:
\begin{definition}[Old $\aleph$-emulation]
$\mu : \Sigma^\star \rightarrow M$ emulates $\den{\cdot}_L$ w.r.t. $=$ iff:
\begin{compactenum}
    \item $\mu \cong_{=} \den{\cdot}_L$ over $\suppinv(S)$, for all $S \subseteq (\Sigma^\star)^2$ \label{cond:isomorphic}
    \item There exists a Turing machine that computes whether $m = m'$ for each $m, m' \in M$ \label{cond:decidable}
    \item There exists a Turing machine with oracle access to $\assert_L$ that computes $\mu$ \label{cond:computable}
\end{compactenum}
\end{definition}
For a set of languages $\mathcal L$, this is equivalent to saying $\mathcal L$ is $\aleph$-emulatable iff, for all $L \in \mathcal L$, there exists an oracle Turing machine $\mu$ and normal Turing machine $\delta$ such that, for all $S \in (\Sigma^\star)^2$, $e, e' \in \suppinv(S)$,
\begin{equation*}
    \den{e}_L = \den{e'}_L \iff \delta \big( \mu_L(e) , \mu_L(e') \big) .
\end{equation*}
This more closely resembles \autoref{def:emulation}, but we will make two slight changes. First, we will change the quantifier order, such that a single $\mu$ must work for every $L \in \mathcal L$. Then, we will grant $\delta$ access to a context $\kappa$, and rephrase the equation to hold over all $\kappa \in (\Sigma^\star)^2$ and $e, e' \in \supp_L(\kappa)$:
\begin{equation*}
    \den{e \mid \kappa}_L = \den{e' \mid \kappa}_L \iff \delta \big( \mu_L(e) , \mu_L(e') \mid \kappa \big) . \label{eq:current}
\end{equation*}
This recovers \autoref{def:emulation}. This version more faithfully reflects the intuitive notion of emulation. The old version required $\mu_L(e)$ to determine how $e$ should evaluate in every possible context. Emulation would not be possible in some cases even with perfect knowledge of $L$. Now, it must just be possible in any context $\kappa$ to compute $\den{e \mid \kappa}_L$ from $\kappa$ and $\mu_L(e)$, which is a weaker standard. Under the new definition, it is \emph{always} possible to emulate a class of languages with one element, assuming $\den{e \mid \kappa}_L$ is computable. An additional improvement is that emulation now applies to all expressions that share a context, whereas before it only targeted expressions with the same support.

\end{document}